\documentclass[11pt]{article}
\pdfoutput=1
\usepackage{fullpage}
\usepackage{amsmath,amsfonts,amsthm,amssymb}
\usepackage{url}
\usepackage{graphicx}
\usepackage{caption} 
\usepackage{algorithm}
\usepackage{algorithmic}
\usepackage{bbm}
\usepackage{bm}
\usepackage{float}
\usepackage{framed}
\usepackage{enumerate}
\usepackage{natbib}
\usepackage{color}
\usepackage{todonotes}
\usepackage[colorlinks=true, linkcolor=blue, urlcolor=blue, citecolor=blue]{hyperref}
\usepackage{multirow}
\usepackage{makecell}

\theoremstyle{plain}
 \newtheorem{theorem}{Theorem}[section]
 \newtheorem{lemma}[theorem]{Lemma}
 \newtheorem{corollary}[theorem]{Corollary}
 \newtheorem{proposition}[theorem]{Proposition}

 \newtheorem{assumption}[theorem]{Assumption}
 \newtheorem*{theorem*}{Theorem}
 \newtheorem*{lemma*}{Lemma}
 \newtheorem*{corollary*}{Corollary}
 \newtheorem*{proposition*}{Proposition}
 \newtheorem*{claim*}{Claim}
 \newtheorem*{fact*}{Fact}
 \newtheorem*{observation*}{Observation}
 \newtheorem*{assumption*}{Assumption}

 \theoremstyle{definition}
 \newtheorem{definition}[theorem]{Definition}
 \newtheorem{remark}[theorem]{Remark}
 
 \newtheorem*{definition*}{Definition}
 \newtheorem*{remark*}{Remark}
 \newtheorem*{example*}{Example}

\title{Faster Projection-free Online Learning}
\date{\today}
\author{  Elad Hazan$^{1,2}$ \qquad Edgar Minasyan$^{1}$ \\
  $^1$ Department of Computer Science, Princeton University \\
  $^2$ Google AI Princeton \\
  \texttt{\{ehazan,minasyan\}@princeton.edu }  \\
 }

\usepackage{url}
\usepackage{graphicx}
\usepackage{algorithm}
\usepackage{algorithmic}
\usepackage{bbm}
\usepackage{bm}
\usepackage{float}
\usepackage{framed}
\usepackage{enumerate}
\usepackage{color}
\usepackage{amsmath,amsfonts,amssymb}
\usepackage{mathtools}
\usepackage{latexsym}
\usepackage{relsize}

\usepackage[capitalise]{cleveref}
\usepackage{xcolor}
\usepackage{dsfont}

\newcommand{\R}{\mathbb{R}}

\newcommand{\A}{\mathcal{A}}

\newcommand{\y}{\ensuremath{\mathbf y}}

\def\x{\mathbf{x}}
\def\y{\mathbf{y}}

\def\v{\mathbf{v}}

\newcommand{\ignore}[1]{}

\DeclareMathAlphabet{\mathbfsf}{\encodingdefault}{\sfdefault}{bx}{n}

\DeclareMathOperator*{\argmax}{arg\,max}

\let\Pr\relax
\DeclareMathOperator{\Pr}{\mathbb{P}}

\newcommand{\E}{\mathbb{E}}

\newcommand{\var}{\mathrm{Var}}

\newcommand{\eps}{\varepsilon}

\renewcommand{\leq}{~\le~}
\renewcommand{\geq}{~\ge~}

\let\oldtfrac\tfrac
\renewcommand{\tfrac}[2]{\smash{\oldtfrac{#1}{#2}}}

\let\nablaold\nabla
\renewcommand{\nabla}{\nablaold\mkern-2.5mu}

\numberwithin{equation}{section}

\begin{document}

\maketitle

\begin{abstract}
In many online learning problems the computational bottleneck for gradient-based methods is the  projection operation. For this reason, in many problems the most efficient algorithms are based on the Frank-Wolfe method, which replaces projections by linear optimization. In the general case, however, online projection-free methods require more iterations than projection-based methods: the best known regret bound scales as $T^{3/4}$. Despite significant work on various variants of the Frank-Wolfe method, this bound has remained unchanged for a decade.

In this paper we give an efficient projection-free algorithm that guarantees $T^{2/3}$ regret for general online convex optimization with smooth cost functions and one linear optimization computation per iteration. As opposed to previous Frank-Wolfe approaches, our algorithm is derived using the Follow-the-Perturbed-Leader method and is analyzed using an online primal-dual framework.

\end{abstract}

\section{Introduction}

In many machine learning problems the decision set is high dimensional or otherwise complex such that even convex optimization over the set is not practical. Such is the case, for example, in matrix learning problems: performing matrix decomposition for very large problems is computationally intensive and super-linear in the sparsity of the input. This renders common algorithms such as projected gradient descent infeasible.  

An alternative methodology which has proven successful in several applications is projection-free online learning. In this model, the access of the learner to the decision set is via a {\it linear optimization} oracle, as opposed to general convex optimization. As an example, linear optimization over matrices amounts to eigenvector computations, which can be carried out in time proportional to the sparsity of the matrices.  

We henceforth consider online algorithms that perform one (or more generally a constant number) linear optimization and/or gradient evaluation per iteration. The reason is that if we do not restrict the number of linear optimizations, we can compute projections up to arbitrary precision and run standard projected online gradient descent which attains optimal regret. We restrict the number of gradient oracle calls since otherwise, in the stochastic setting, one can evaluate the real gradient up to arbitrary precision and run standard offline Frank-Wolfe which attains optimal convergence rate. This defeats the purpose of creating efficient algorithms.

\begin{definition}
The {\bf oracle complexity} of a projection free algorithm is the number of linear optimizations and gradient evaluations {\it per iteration}. We say that a projection-free algorithm is  {\bf oracle-efficient} if its oracle complexity is constant.
\end{definition}


Oracle-efficient projection-free methods can have a significant running time advantage for certain structured problems in which linear optimization is more efficient than projections. This has spurred significant research in recent years on projection-free methods and the Frank-Wolfe algorithm. However, despite a decade-long search, the best known oracle-efficient projection-free online algorithm attains a regret bound that scales as  $T^{3/4}$, where $T$ is the number of iterations \footnote{We omit the $O$-notation in the introduction to make the exposition cleaner. In this case, $T^{3/4}$ hides constants including the norm of the gradients, diameter of the decision set, and more. See \cite{hazan2016introduction} for more details.}. 
This method, the Online Frank-Wolfe (OFW) algorithm  \citep{HazanKa12}, attains the lowest oracle complexity over all iterations, even if we include projection-free methods that are not oracle-efficient.

The $T^{3/4}$ bound is particularly striking when compared to stochastic projection-free optimization. In this setting, it is possible to obtain $T^{2/3}$ regret for smooth stochastic projection-free optimization by the so-called blocking technique: grouping several game iterations into one and thereby changing the decision less often, and its variants \citep{merhav2002sequential,chen18csubmodular}. The optimal rate of $\sqrt{T}$ is more challenging to obtain as given in \cite{LanZh14}. 
However, we are unaware of an improvement to the $T^{3/4}$ rate even in the stochastic non-smooth case, when only constantly many linear optimizations and gradient evaluations per iteration are allowed.


\subsection{Our Results}

Our main result is an efficient randomized algorithm that improves the state-of-the-art in general projection-free online optimization with smooth loss functions. The expected regret of this algorithm scales as $T^{2/3}$, with only one linear optimization computation per iteration. We then extend the analysis of this algorithm to show that it attains the same regret bound with high probability. 
Our main results are summarized by the informal theorem below, with the exact dependence on smoothness and other relevant problem parameters detailed in later sections.


\begin{theorem}
There exists an efficient algorithm for online convex optimization (see Algorithm \ref{alg:ospf}) with smooth loss functions that is projection-free, performs only one linear optimization computation per iteration, and guarantees an expected regret of $ O(  T^{2/3})  $. Furthermore, the algorithm guarantees a regret of $ \tilde{O}( T^{2/3}  \log \frac{1}{\sigma} ) $ with probability at least $1 - \sigma$.
\end{theorem}

\begin{center}
\begin{table}[h]
\centering
\begin{tabular}{|c|c|c|c|}
\hline
\multicolumn{4}{|c|}{ Regret with \textbf{constant oracle complexity} }\\
\hline
Loss Function & Setting & Regret
& Reference\\
\hline
Non-Smooth  &  Online & $T^{3/4}$ & \cite{HazanKa12} \\ 
Smooth  & Online  &   $T^{2/3}$ &  {\bf this paper} \\ 
Smooth   &  Stochastic &  $ T^{1/2}$ & \cite{xie2019efficient} \\ 
Smooth   &  Offline  & $ \log T $ &  \cite{FrankWo56} \\ 
\hline
\end{tabular}
\caption{Comparison of known regret bounds for projection-free optimization. We only refer to algorithms that have constant oracle complexity
}\label{table:result_summary}
\end{table}
\end{center}

\paragraph{Techniques.} Our algorithm is not based on the Frank-Wolfe method, but rather a version of the Follow-the-Perturbed-Leader (FPL) method \citep{kalai2005efficient}. 
It was already established in \cite{hazan2016introduction} that a deterministic version of FPL works for online convex optimization. This version computes the expected point FPL plays at every iteration.
In order to convert this algorithm to an efficient projection-free method, two main challenges arise: 
\begin{enumerate}
    \item Estimating the expectation by sampling FPL points via linear optimization creates time dependence between iterations, since the gradient is taken at a point which depends on all previous iterations. This means that a small error in one iteration potentially propagates to all future iterations. 
    
    \item The number of linear optimization evaluations to estimate the mean of the FPL algorithm up to $\eps$ accuracy is $O(\frac{1}{\eps^2})$.  This dependence is not sufficient to improve the previously best $T^{3/4}$ regret bound with only constantly many linear optimization computations per iteration.
    
\end{enumerate}

To overcome the above issues we require two tools, that are new to the analysis of randomized online methods. First, we use the online primal-dual methodology of \cite{shalev2007primal}. This allows us to avoid the error-propagation caused by random estimation of the mean, and could be a technique of independent interest. 

The second tool is using the smoothness of the loss functions to leverage not only the estimation proximity but also the fact that the estimation is unbiased. This is executed by switching gradients at nearby points which are not too far off due to the Lipschitz property of the gradients of smooth functions.

\paragraph{Paper structure.}  In the next subsection we discuss related work, and then move to describe preliminaries, including tools necessary for the online primal-dual analysis framework. In section \ref{sec:statement}, we state the main algorithm and formally state our main theorems including precise constants. In section \ref{sec:determinstic} we state the deterministic FPL algorithm and analyze it using the primal-dual framework to illustrate its versatility in handling error propagation. We then use unbiased estimation and smoothness in section \ref{sec:samples} to derive the first main theorem. In section \ref{sec:reduction}, we provide a reduction of the algorithm to the setting of one linear optimization step per iteration. High probability bounds are given in detail in section \ref{sec:highprob} and derived in the appendix along with other miscellaneous proofs.

\ignore{
\begin{center}
\begin{table}[h]
\centering
\begin{tabular}{|c|c|c|}
\hline
\multicolumn{3}{|c|}{\makecell{ \textbf{$T_\eps$ complexity}}}\\
\hline
Setting & $T_\eps$ rate & Reference\\
\hline
$\beta$-Smooth  Convex Loss &   $\eps^{-4}$ & \cite{HazanKa12} \\ 
& $\eps^{-3} $ &  (Theorem~\ref{})\\
\hline
\end{tabular}
\caption{Summary of our results}\label{table:result_summary}
\end{table}
\end{center}
}

\subsection{Related Work}

In recent years the projection-free learning and optimization literature has seen a resurgence of results. We separate the related work into the broad categories below. 

\paragraph{Projection-free offline optimization.} The starting point for our line of work is the seminal paper of \cite{FrankWo56}, who apply the conditional gradient method for smooth optimization over polyhedral sets. This was extended to semi-definite programming in \cite{hazan2008sparse}, and to general convex optimization in \cite{Jaggi13}. This algorithm requires $\frac{\beta}{\eps}$ linear optimization steps to find an $\eps$-approximate solution for a $\beta$-smooth function, optimal with no other assumptions.  

A significant advancement in projection-free methods was obtained by \cite{GarberHa13}, who give an algorithm that requires only $\log \frac{1}{\eps}$ linear optimization steps for strongly convex and smooth functions over polyhedral sets. Data-dependent bounds for the spectahedron were obtained by \cite{garber2016faster,allen2017linear}. 

Projection-free optimization on non-smooth objective functions is typically performed via various smoothing techniques. The optimal complexity of linear optimization calls in this case is $\frac{1}{\eps^2}$ \citep{lan2013complexity}. Several algorithms attain nearly optimal rates as in \cite{lan2013complexity, argyriou2014hybrid, pierucci2014smoothing}. 

\paragraph{Projection-free online learning.}

The online variant of the Frank-Wolfe algorithm that applies to general online convex optimization was given in \cite{HazanKa12}. This method attains $T^{3/4}$ regret for the general OCO setting, with only one linear optimization step per iteration \footnote{If arbitrarily many linear optimization steps are allowed, the projections can be computed and this regret can be improved to $\sqrt{T}$.}. 

For OCO over polyhedral sets, an implication of the result of \cite{GarberHa13} is an efficient $\sqrt{T}$ regret algorithm with only one linear optimization step per iteration, as well as $\log T$ regret for strongly convex online optimization. Recently \cite{pmlr-v89-levy19a} give an efficient projection-free online learning algorithm for smooth sets that devises a new fast projection operation for such sets and attains the optimal $\sqrt{T}$ regret for convex and $\log T$ regret for strongly convex online optimization. 

Without further assumptions, the OFW method in \cite{HazanKa12} attains the best known bounds for general online convex optimization. To the best of our knowledge, our $T^{2/3}$ regret is the first to improve in this general OCO setting for smooth functions.

\paragraph{Projection-free stochastic optimization.}
An important application of projection-free optimization is in the context of supervised learning and the optimization problem of  empirical risk minimization. In this setting, there are more techniques that can be applied to further accelerate optimization, as compared to the online setting, most notably variance reduction.

This requires more careful accounting of the actual operations that the algorithms perform, including counting the number of full-gradient computations, stochastic gradient evaluations, linear optimizations, and projections. There have been a multitude of algorithms suggested that attain various tradeoffs of the various computations, and have different merits/caveats. The reader is referred to the vast literature on stochastic projection-free methods, including the recent papers of \cite{LanZh14,hazan2016variance, chen18csubmodular, hassani2019stochastic, xie2019efficient, yurtsever19conditional, zhang2019sample}.

\section{Problem Setting} We consider a classical online convex optimization framework as an iterative game between a player and an adversary. At each iteration $t \in \mathbb{N}$, the player chooses an action $\x_t$ from the constrained set $\mathcal{K} \subset \R^d$ of permissible actions while the adversary simultaneously chooses a loss function $f_t : \mathcal{K} \to \R$ that determines the loss the player will occur for the action $\x_t$. The performance metric for such settings is the notion of regret -- the difference between the cumulative loss suffered throughout $T$ iterations of the online game and the overall loss for the best fixed action in hindsight:
\begin{equation}\label{eq:regretdef}
    \mathcal{R}_T = \sum_{t=1}^T f_t(\x_t) - \min_{\x \in \mathcal{K}} \sum_{t=1}^T f_t(\x)
\end{equation}
For a given online algorithm $\mathcal{A}$, we denote $\mathcal{R}_T(\mathcal{A})$ to be the regret after $T$ iterations and, in the case when $\mathcal{A}$ is a randomized algorithm, use expected regret $\E[\mathcal{R}_T(\mathcal{A})]$ as the performance metric. In this work, the adversary has no computational or information restrictions as long as it chooses $f_t$ simultaneously to the player choosing $\x_t$, i.e. we operate in the adaptive adversarial setting.

Before proceeding to the main results, we formalize several notations and assumptions preserved throughout the paper. We discuss and, if necessary, derive our additional/modified assumptions, made for simplicity in the analysis, without touching upon the conventional standards established in the community, explanations of which can be found in the extensive literature (e.g. see \cite{hazan2016introduction}). Throughout this work the use of norm $\| \cdot \|$ refers to the standard Euclidean norm unless stated otherwise and $\mathbb{B} = \{ v \in \R^d, \, \| v \| \leq 1 \}$ denotes the unit ball. Given any sequence $\{a_t\}_{t \in \mathbb{N}}$, by abuse of notation we denote $a_{l:k}$ as a shorthand for the indexed sum $\sum_{i=l}^k a_i = a_{l:k}$ or the indexed union $\cup_{i=l}^k a_i = a_{l:k}$.
\begin{assumption}\label{assume:set}
The constrained action set $\mathcal{K} \subset \R^d$ is a convex and compact set. Moreover, all the points in the set have bounded norms, i.e. $\| \x \| \leq D, \, \forall \x \in \mathcal{K}$.
\end{assumption}
\begin{assumption}\label{assume:function}
For each iteration $t \in \mathbb{N}$, the loss functions $f_t : \mathcal{K} \to \R$ are convex, differentiable and have bounded gradient norms $\forall \x \in \mathcal{K}, \, \| \nabla f_t(\x) \| \leq G$.
\end{assumption}
The convention in OCO is to simply assume a bounded diameter for the set $\mathcal{K}$ instead of the norm bound. However, it is straightforward to derive the above formulation the following way: given an arbitrary point in the set $\x' \in \mathcal{K}'$ consider the shifted set $\mathcal{K} = \{ \x - \x' \, | \, \x \in \mathcal{K} \}$; the diameter bound of $\mathcal{K}'$ implies the bounded norms of the points in $\mathcal{K}$ while properties such as convexity and compactness are preserved through shifts. The convexity and bounded gradient norm assumptions for the loss functions are part of the standard throughout literature, while differentiability of the loss functions is assumed for simplicity and can be avoided by using subgradients instead.
\begin{definition}
The Fenchel dual of a function $f : \mathcal{K} \to \R$ with domain $\mathcal{K} \subset \R^d$ is defined as
\begin{equation}\label{eq:fenchel}
    \forall \y \in \R^d, \quad f^*(\y) = \sup_{\x \in \mathcal{K}} \{ \langle \y, \x \rangle - f(\x) \}
\end{equation}
\end{definition}
The definition and some properties of Fenchel duality are given for completeness as the concept of a Fenchel dual will be crucial in the analysis of presented algorithms. If the function $f$ is convex, then its Fenchel dual $f^*$ is also convex, and the Fenchel-Moreau theorem gives biconjugacy, i.e. the dual of a dual is equal to the function itself $(f^*)^* = f$. In this case, it is essential to note that $f(\x) = \sup_{\y \in \R^d} \{ \langle \x, \y \rangle - f^*(\y) \}$ directly implies $\nabla f(\x) = \arg\sup_{\y \in \R^d} \{ \langle \x, \y \rangle - f^*(\y) \}$, which is well-defined, when $f$ is differentiable.
\paragraph{Linear Optimization Oracle.}
A linear optimization oracle, along with a value oracle, over the constraint set $\mathcal{K}$ is provided to the player, defined as
\begin{equation}\label{eq:loo}
    \forall \y \in \R^d, \quad \mathcal{O}_{\mathcal{K}}(\y) = \argmax_{\x \in \mathcal{K}} \langle \y, \x \rangle, \quad \mathcal{M}_{\mathcal{K}}(\y) = \max_{\x \in \mathcal{K}} \langle y, x \rangle
\end{equation}

The reliance on linear optimization is the key motivation of the paper. This work concerns itself with the special case of online convex constrained optimization where the operation of projection to the set $\mathcal{K}$, as a problem of quadratic optimization, has a significantly higher computational cost than the linear optimization. In such cases, the use of the projected Online Gradient Descent (OGD) \citep{zinkevich2003online} that achieves an optimal regret bound $O(\sqrt{T})$ with respect to $T$ is not always preferred to methods that bypass projection and use linear optimization instead. It is worth to notice that the existence of only $\mathcal{O}_{\mathcal{K}}(\cdot)$ is enough since $\mathcal{M}_{\mathcal{K}}(\y) = \langle \y, \mathcal{O}_{\mathcal{K}}(\y) \rangle$ and $\nabla \mathcal{M}_{\mathcal{K}}(\y) = \mathcal{O}_{\mathcal{K}}(\y)$. Moreover, the function $\mathcal{M}_{\mathcal{K}}(\cdot)$ is convex and Lipschitz as suggested below.
\begin{lemma}\label{lemma:lipschitz}
The linear value oracle $\mathcal{M}_{\mathcal{K}} : \R^d \to \R$ is convex and $D$-Lipschitz, i.e.
\begin{equation}
    \forall \y_1, \y_2 \in \R^d, \quad |\mathcal{M}_{\mathcal{K}}(\y_1) - \mathcal{M}_{\mathcal{K}}(\y_2)| \leq D \| \y_1 - \y_2 \|
\end{equation}
\end{lemma}


\section{Algorithm and Main Theorem}  \label{sec:statement}
The algorithm we propose is fairly straightforward, and the main hurdle lies in the analysis. The seminal work of \cite{kalai2005efficient} introduces the Follow-the-Perturbed-Leader (FPL) online algorithm that obtains optimal $O(\sqrt{T})$ regret for linear loss functions. A more general version of FPL that applies expectations over the perturbations at each iteration extends the result to general convex functions \citep{hazan2016introduction}. Our algorithm mimics the expected FPL replacing the computationally expensive expectations with empirical averages of i.i.d. samples. It is presented in detail in Algorithm \ref{alg:main}. The following theorem states the convergence guarantees for both general convex and smooth convex loss functions.

\begin{algorithm}
\caption{Sampled Follow-the-Perturbed-Leader Algorithm, $\mathcal{A}_1$}
\label{alg:main}
\begin{algorithmic}
\STATE \textbf{Input:} constraint set $\mathcal{K}$, number of rounds $T$, perturbation parameter $\delta$, number of samples $m$, linear optimization oracle $\mathcal{O}_{\mathcal{K}}(\cdot)$
\FOR {$t=1$ \TO $T$}
\STATE sample $\v_t^j \sim \mathbb{B}$ uniformly for $j = 1, \dots, m$
\STATE denote $\x_t^j = \mathcal{O}_{\mathcal{K}}(- \tilde{\nabla}_{1:t-1} + \frac{1}{\delta} \cdot \v_t^j)$ for $j=1, \dots, m$
\STATE play $\tilde{\x}_t = \frac{1}{m} \sum_{j=1}^m \x_t^j$
\STATE observe $f_t$, denote $\tilde{\nabla_t} = \nabla f_t(\tilde{\x}_t)$
\ENDFOR
\end{algorithmic}
\end{algorithm}

\begin{theorem}\label{thm:main}
Given that the Assumptions \ref{assume:set} and \ref{assume:function} hold, Algorithm \ref{alg:main}, for general convex loss functions, obtains an expected regret of
\begin{equation}\label{eq:genregret}
    \E \left[\sum_{t=1}^T f_t(\tilde{\x}_t) \right] \leq \min_{\x \in \mathcal{K}}\{ \sum_{t=1}^T f_t(\x) \} + 2 D / \delta + \delta D G^2 \cdot d T / 2 + \frac{2 G D T}{\sqrt{m}}
\end{equation}
If the convex loss functions are also $\beta$-smooth then the expected regret bound becomes
\begin{equation}\label{eq:smoothregret}
    \E \left[\sum_{t=1}^T f_t(\tilde{\x}_t) \right] \leq \min_{\x \in \mathcal{K}}\{ \sum_{t=1}^T f_t(\x) \} + 2 D / \delta + \delta D G^2 \cdot d T / 2 + \frac{4 \beta D^2 T}{m}
\end{equation}
\end{theorem}

\begin{remark}\label{rem:regretbounds}
It follows from Theorem \ref{thm:main} that Algorithm \ref{alg:main} attains an expected regret of $\E[\mathcal{R}_T(\mathcal{A}_1)] = O(\sqrt{T})$ with the parameter choices of $\delta = O(1/\sqrt{T})$ and $m = O(T), O(\beta \sqrt{T})$ for general convex and smooth convex functions respectively. In particular, this restores the original result of the FPL method attaining $O(\sqrt{T})$ regret with $m=1$ sample per iteration for linear, $\beta=0$, loss functions shown in \cite{kalai2005efficient}.
\end{remark}

\begin{algorithm}
\caption{Online Smooth Projection Free (OSPF) Algorithm, $\mathcal{A}_{\text{OSPF}}$}
\label{alg:ospf}
\begin{algorithmic}
\STATE \textbf{Input:} constraint set $\mathcal{K}$, number of rounds $T$, perturbation parameter $\delta$, block size $k$, linear optimization oracle $\mathcal{O}_{\mathcal{K}}(\cdot)$
\STATE pick arbitrary $\x_0 \in \mathcal{K}$, denote $\nabla_0 = \bm{0}$
\FOR {$t=1$ \TO $T$}
\IF{$t \mod k \neq 0$}
    \STATE play $\x_t = \x_{t-1}$
    \STATE observe $f_t$, denote $\nabla_t = \nabla f_t(\x_t)$
\ELSE
    \STATE sample $\v_{t-k+j} \sim \mathbb{B}$ uniformly for $j = 1, \dots, k$
    \STATE denote $\x_t^j = \mathcal{O}_{\mathcal{K}}(- \nabla_{0:t-1} + \frac{1}{\delta} \cdot \v_{t-k+j})$ for $j=1, \dots, k$
    \STATE play $\x_t = \frac{1}{k} \sum_{j=1}^k \x_t^j$
    \STATE observe $f_t$, denote $\nabla_t = \nabla f_t(\x_t)$
\ENDIF
\ENDFOR
\end{algorithmic}
\end{algorithm}

\begin{corollary}\label{cor:reduction}
The expected regret bound of Algorithm \ref{alg:main} for general convex loss functions implies $O(G D \sqrt{d} T^{3/4})$ expected regret when using one linear optimization step per iteration. The analogous reduction induces the OSPF algorithm given in Algorithm \ref{alg:ospf} that attains $O(D (G \sqrt{d} + \beta D) T^{2/3})$ expected regret for smooth functions with one linear optimization step per iteration.
\end{corollary}

\section{The Case of Unlimited Computation} \label{sec:determinstic}
In an ideal scenario the player would be given unrestrained computational power along with access to the linear optimization oracle. Then the expected FPL method, as a projection-free online algorithm, is known to obtain $O(\sqrt{T})$ regret bound optimal with respect to $T$ for general convex loss functions. The exact algorithm is spelled out in Algorithm \ref{alg:eftpl}. The original analysis follows the standard recipe of online learning literature coined by \cite{kalai2005efficient}: no regret for Be-The-Leader -- the algorithm suffers no regret if it is hypothetically one step ahead of the adversary, i.e. uses $\x_{t+1}$ for the loss function $f_t$; stability -- the predictions of consecutive rounds $\x_t, \x_{t+1}$ are not too far apart from each other \citep{hazan2016introduction}. We provide an alternative approach developed by \cite{shalev2007primal} that is based on duality and enables the further analysis of Algorithm \ref{alg:main}.
\begin{algorithm}
\caption{Expected Follow-the-Perturbed-Leader Algorithm, $\mathcal{A}_3$ \label{alg:eftpl}}
\begin{algorithmic}
\STATE \textbf{Input:} constraint set $\mathcal{K}$, number of rounds $T$, perturbation parameter $\delta$, linear optimization oracle $\mathcal{O}_{\mathcal{K}}(\cdot)$
\FOR {$t=1$ \TO $T$}
\STATE compute $\x_t = \E_{\v \sim \mathbb{B}}[\mathcal{O}_{\mathcal{K}}(- \nabla_{1:t-1} + \frac{1}{\delta} \cdot \v)]$
\STATE play $\x_t$, observe $f_t$, denote $\nabla_t = \nabla f_t(\x_t)$
\ENDFOR
\end{algorithmic}
\end{algorithm}

\begin{theorem}\label{thm:eftpl}
Given that Assumptions \ref{assume:set} and \ref{assume:function} hold, Algorithm \ref{alg:eftpl} suffers $\mathcal{R}_T(\mathcal{A}_3) = O(\sqrt{T})$ regret.
\end{theorem}
\begin{proof}
The proof is based on duality when one considers the following optimization problem
\begin{equation}\label{eq:primal}
    \min_{\x \in \mathcal{K}} \{ h_{\delta}(\x) + \sum_{t=1}^T f_t(\x) \}
\end{equation} 
which resembles the loss suffered by the best-in-hindsight fixed action. The dual objective, that is to be maximized, can be obtained using Lagrange multipliers and is given by (see \cite{shalev2007primal} for details)
\begin{equation}\label{eq:dualobj}
    \mathcal{D}(\bm{\lambda}_1, \dots, \bm{\lambda}_T) = - h_{\delta}^*(- \bm{\lambda}_{1:T}) - \sum_{t=1}^T f_t^*(\bm{\lambda}_t)
\end{equation}
The term $h_{\delta}(\cdot)$ serves as regularization and is defined implicitly through its Fenchel conjugate $h_{\delta}^*(\y) = \E_{\v \sim \mathbb{B}}[\mathcal{M}_{\mathcal{K}}(\y + \frac{1}{\delta} \cdot \v)]$, a stochastic smoothing of the value oracle, that is $\delta dD$-smooth according to the following lemma and the fact that $\mathcal{M}_{\mathcal{K}}(\cdot)$ is $D$-Lipschitz. 
\begin{lemma}\label{lemma:randsmoothing}
The function $\hat{g}(\y) = \E_{\v \sim \mathbb{B}}[g(\y + \frac{1}{\delta} \cdot \v)]$ is $\delta d L$-smooth given $g : \R^d \to \R$ is an $L$-Lipschitz function.
\end{lemma}
The duality gap suggests that the objective \eqref{eq:dualobj} is upper bounded by \eqref{eq:primal} for any values of $\bm{\lambda}_t, t=1, \dots, T$ hence the goal is to upper bound the online cumulative loss by \eqref{eq:dualobj}. To achieve this, take $\bm{\lambda}_t = \nabla_t = \nabla f_t(\x_t)$ for all $t \in [T]$ where the action $\x_t$ is chosen according to Algorithm \ref{alg:eftpl}. Denote the incremental difference as $\Delta_t = \mathcal{D}(\nabla_1, \dots, \nabla_t, \bm{0}, \dots, \bm{0}) - \mathcal{D}(\nabla_1, \dots, \nabla_{t-1}, \bm{0}, \dots, \bm{0})$ and notice that the dual can be written as $\mathcal{D}(\nabla_1, \dots, \nabla_T) = \sum_{t=1}^T \Delta_t + \mathcal{D}(\bm{0}, \dots, \bm{0})$. For each $t \in [T]$,
\begin{align}
    \Delta_t &= -\left[ h_{\delta}^*(-\nabla_{1:t}) - h_{\delta}^*(- \nabla_{1:t-1}) \right] - f_t^*(\nabla_t) + f_t^*(\bm{0}) \geq \tag*{smoothness of $h_{\delta}^*(\cdot)$} \nonumber \\
    &\geq \langle \nabla_t, \nabla h_{\delta}^*(- \nabla_{1:t-1}) \rangle - \frac{\delta dD}{2} \| \nabla_t \|^2 - f_t^*(\nabla_t) + f_t^*(\bm{0})= \tag*{definition of $\x_t$} \nonumber \\
    &= \langle \nabla_t, \x_t \rangle - f_t^*(\nabla_t) - \frac{\delta dD}{2} \| \nabla_t \|^2 + f_t^*(\bm{0}) = f_t(\x_t) - \frac{\delta dD}{2} \| \nabla_t \|^2 + f_t^*(\bm{0})\label{eq:mainbound}
\end{align}
where we use the fact that the action $\x_t$ from Algorithm \ref{alg:eftpl} can alternatively be expressed as $\x_t = \nabla h_{\delta}^*(-\nabla_{1:t-1})$ and the Fenchel dual identity $\langle \nabla_t, \x_t \rangle - f_t^*(\nabla_t) = f_t(\x_t)$ for convex $f_t(\cdot)$. The obtained inequality \eqref{eq:mainbound} quantifies how much regret an action contributes at a given iteration $t \in [T]$ detached from the rest of the rounds of the game. Such a property of the analysis ends up being crucial in showing the regret bounds further in this work. Note that by definition $\mathcal{D}(\bm{0}, \dots, \bm{0}) = - h_{\delta}^*(\bm{0}) - \sum_{t=1}^T f_t^*(\bm{0})$ which gives the identity $\sum_{t=1}^T \Delta_t - \sum_{t=1}^T f_t^*(\bm{0}) = \mathcal{D}(\nabla_1, \dots, \nabla_T) + h_{\delta}^*(\bm{0})$. Thus,  sum up \eqref{eq:mainbound} for all $t = 1, \dots, T$ to bound the online cumulative loss is by
\begin{equation}\label{eq:onlineloss}
    \sum_{t=1}^T f_t(\x_t) \leq \mathcal{D}(\nabla_1, \dots, \nabla_T) + h_{\delta}^*(\bm{0}) + \frac{\delta dD}{2} \sum_{t=1}^T \| \nabla_t \|^2
\end{equation}
The bound given by \eqref{eq:onlineloss} and the duality gap of the primal \eqref{eq:primal} provide the necessary ingredients to conclude the $O(\sqrt{T})$ regret bound. All that is left are technical details to reach the bound using the assumptions of the given setup. First, by definition $\mathcal{M}_{\mathcal{K}}(\bm{0}) = 0$ which implies, by Lipschitzness of $\mathcal{M}_{\mathcal{K}}(\cdot)$, that $| \mathcal{M}_{\mathcal{K}}(\frac{1}{\delta} \cdot \v) | \leq D \| \v \| / \delta \leq D / \delta$ for any $\v \in \mathbb{B}$ so $h_{\delta}^*(\bm{0}) \leq D / \delta$. Second, the primal expression in \eqref{eq:primal} can be related to the best loss in hindsight the following way
\begin{equation}
    \min_{\x \in \mathcal{K}} \{ h_{\delta}(\x) + \sum_{t=1}^T f_t(\x) \} \leq h_{\delta}(\x^*) + \sum_{t=1}^T f_t(\x^*) \leq \min_{\x \in \mathcal{K}} \sum_{t=1}^T f_t(\x) + \max_{\x \in \mathcal{K}} h_{\delta}(\x)
\end{equation}
where $\x^*$ is the optimal action in hindsight, i.e. the minimizer of $\sum_{t=1}^T f_t(\cdot)$ over $\mathcal{K}$. Moreover, notice that for any $\x \in \mathcal{K}, \y \in \R^d$ the expression $\langle \x, \y \rangle - h_{\delta}^*(\y) = \E_{\v \sim \mathbb{B}}[\langle \x, \y \rangle - \max_{\x' \in \mathcal{K}} \langle \x', \y + \frac{1}{\delta} \cdot \v \rangle]$ can be bounded as follows: for each $\v \in \mathbb{B}$ the expression inside the expectation is bounded $\langle \x, \y \rangle - \max_{\x' \in \mathcal{K}} \langle \x', \y + \frac{1}{\delta} \cdot \v \rangle \leq \langle \x, \y \rangle - \langle \x, \y + \frac{1}{\delta} \cdot \v \rangle = \langle \x, -\frac{1}{\delta} \cdot \v \rangle \leq \| \x \| \| \v \| / \delta \leq D / \delta$, hence for any $\x \in \mathcal{K}$ the bound $h_{\delta}(\x) \leq D / \delta$ holds. Finally, according to our assumptions the loss gradients are bounded in norm, i.e. $\| \nabla_t \| \leq G$. Combining the aforementioned properties with \eqref{eq:onlineloss} along with the fact that $\mathcal{D}(\nabla_1, \dots, \nabla_T)$ is upper bounded by \eqref{eq:primal} due to the duality gap, we conclude the desired inequality
\begin{equation*}
    \sum_{t=1}^T f_t(\x_t) - \min_{\x \in \mathcal{K}} \sum_{t=1}^T f_t(\x) \leq \max_{\x \in \mathcal{K}} h_{\delta}(\x) + h_{\delta}^*(\bm{0}) + \frac{\delta dD}{2} \sum_{t=1}^T \| \nabla_t \|^2 \leq 2D / \delta + \frac{\delta dD}{2} G^2 T
\end{equation*}
yielding the regret bound $\mathcal{R}_T(\mathcal{A}_3) \leq 2 G D \sqrt{d T} = O(\sqrt{T})$ with the optimal choice of the regularization parameter $\delta = 2 / G \sqrt{d T}$.
\end{proof}
\begin{remark}\label{rem:assumptions}
It is essential to note how each property given in Assumptions \ref{assume:set} and \ref{assume:function} was used in the proof above. The convexity of the constraint set $\mathcal{K}$ allows the action $\x_t$, as an expectation of points in the set $\mathcal{K}$, to be a permissible action as well. Given compactness of $\mathcal{K}$, we interchange the use of supremum and maximum of bounded expressions at various points throughout. The norm bound $D$ of the set $\mathcal{K}$ is used in showing that $\mathcal{M}_{\mathcal{K}}(\cdot)$ is $D$-Lipschitz and bounding several regularization terms. In terms of the loss functions, the convexity of $f_t(\cdot)$, as well as $\mathcal{M}_{
\mathcal{K}}(\cdot)$, allows us to use the Fenchel-Moreau theorem (continuity is implied by differentiability) while the gradient norm bound is simply used in the last stage of obtaining the regret bound. 
\end{remark}

\section{Oracle Efficiency via Estimation} \label{sec:samples}
The results in section \ref{sec:determinstic} suggest that Algorithm \ref{alg:eftpl}, known as expected FPL, possesses the features desired in this work -- it is both online and projection-free -- and obtains an optimal regret bound of $O(\sqrt{T})$. However, it is computationally intractable due to the expectation term given in the definition of the action $\x_t$. In this section, we remedy this issue and explore the scenario where the actions played during the online game are random estimators of the mean. In particular, we propose to simply take the empirical average of $m$ i.i.d. samples instead of the expectation itself as described in Algorithm \ref{alg:main}.

It is essential to note that Algorithm \ref{alg:main} has a computational efficiency of $m \cdot T$ calls to the linear optimization oracle $\mathcal{O}_{\mathcal{K}}(\cdot)$ as the rest of the computation is negligible in comparison. The main theorem of this work, Theorem \ref{thm:main}, indicates the performance of the algorithm in terms of expected regret for (i) general convex loss functions and (ii) smooth convex loss functions, respectively. Given the duality approach to analyzing online algorithms demonstrated in the previous section, the sampled FPL algorithm can now be analyzed to prove the bounds stated in Theorem \ref{thm:main}. In particular, the following lemma demonstrates that each estimation from Algorithm \ref{alg:main} contributes to the regret in a disjoint fashion, i.e. there is no error propagation through time.

\begin{lemma}\label{lemma:prelimregret}
Suppose the Assumptions \ref{assume:set} and \ref{assume:function} hold and denote $\hat{\x}_t = \E_{\v \sim \mathbb{B}} \left[ \mathcal{O}_{\mathcal{K}}(-\tilde{\nabla}_{1:t-1} + \frac{1}{\delta} \cdot \v) \right]$ for all $t \in [T]$ with $\tilde{\x}_t$ and $\tilde{\nabla}_t$ as defined in Algorithm \ref{alg:main}. Then, the regret of the algorithm is bounded as follows
\begin{equation}\label{eq:newregretbnd}
    \sum_{t=1}^T f_t(\tilde{\x}_t) \leq \min_{\x \in \mathcal{K}}\{ \sum_{t=1}^T f_t(\x) \} + \mathcal{R}_T(\mathcal{A}_3) + \sum_{t=1}^T \langle \tilde{\nabla}_t , \hat{\x}_t-\tilde{\x}_t \rangle
\end{equation}
\end{lemma}

\begin{proof}
Follow the same proof structure as in the proof of Theorem \ref{thm:eftpl} by considering \eqref{eq:primal}, \eqref{eq:dualobj} as the primal and dual objectives. Consider $\bm{\lambda}_t = \tilde{\nabla}_t = \nabla f_t(\tilde{\x}_t)$ and denote the incremental difference as $\Delta_t = \mathcal{D}(\tilde{\nabla}_1, \dots, \tilde{\nabla}_t, \bm{0}, \dots, \bm{0}) - \mathcal{D}(\tilde{\nabla}_1, \dots, \tilde{\nabla}_{t-1}, \bm{0}, \dots, \bm{0})$. The main component of the proof is showing that $\Delta_t$ for each $t \in [T]$ can be roughly seen as an upper bound on the loss $f_t(\tilde{\x}_t)$ suffered at iteration $t$. 

First note that the played actions $\tilde{\x}_t$ are not, in fact, unbiased estimators of the original $\x_t$; instead denote the expectations by $\hat{\x}_t = \E_{\xi_t}[\tilde{\x}_t] = \E_{\v_t^j}[\x_t^j]$ where $\xi_t = \{\v_t^1, \dots, \v_t^m\}$ comprises the randomness used at iteration $t \in [T]$. For all $t > 1$, the quantity $\hat{\x}_t$ is different from $\x_t$ in that it uses the gradients at the points $\tilde{\x}_1, \dots, \tilde{\x}_{t-1}$ instead of $\x_1, \dots, \x_{t-1}$ and such difference can potentially increase with $t$. In other words, one is defined as $\hat{\x}_t = \E_{\v \sim \mathbb{B}}[\mathcal{O}_{\mathcal{K}}(-\tilde{\nabla}_{1:t-1} + \frac{1}{\delta} \cdot \v)] = \nabla h_{\delta}^*(-\tilde{\nabla}_{1:t-1})$ while the other is equal to $\x_t = \nabla h_{\delta}^*(-\nabla_{1:t-1})$. Hence, the action sequences of $\tilde{\x}_1, \dots, \tilde{\x}_T$ and $\x_1, \dots, \x_T$ can behave quite differently  and one cannot analyze the former based on results about the latter. However, the duality approach enables us to directly analyze the actions of Algorithm \ref{alg:main}. In particular, lower bound the quantity $\Delta_t$ using the smoothness of $h_{\delta}^*(\cdot)$, as done in the proof of Theorem \ref{thm:eftpl}:
\begin{equation}\label{eq:deltabnd}
    \Delta_t \geq \langle \tilde{\nabla}_t, \hat{\x}_t \rangle - \frac{\delta d D}{2} \| \tilde{\nabla}_t \|^2 - f_t^*(\tilde{\nabla}_t) + f_t^*(\bm{0}) = f_t(\tilde{\x}_t) - \frac{\delta d D}{2} \| \tilde{\nabla}_t \|^2 + f_t^*(\bm{0}) + \langle \tilde{\nabla}_t, \hat{\x}_t - \tilde{\x}_t \rangle
\end{equation}
The obtained inequality \eqref{eq:deltabnd} resembles the analogous bound \eqref{eq:mainbound} in the unlimited computation case with the extra term $\langle \tilde{\nabla}_t, \hat{\x}_t - \tilde{\x}_t \rangle$ that can be seen as accounting for the estimation error. This shows that at a given iteration $t \in [T]$ the additional regret is suffered {\it only} at the expense of the current action choice, $\tilde{\x}_t$ instead of $\hat{\x}_t$, while {\it ignoring} the optimality of the previous choices $\tilde{\x}_1, \dots, \tilde{\x}_{t-1}$. We proceed with the proof by summing up $\eqref{eq:deltabnd}$ for all $t = 1, \dots, T$ and using the following facts: by definition $\sum_{t=1}^T \Delta_t = \mathcal{D}(\tilde{\nabla}_1, \dots, \tilde{\nabla}_T) - \mathcal{D}(\bm{0}, \dots, \bm{0})$ and  $\mathcal{D}(\bm{0}, \dots, \bm{0}) = - h_{\delta}^*(\bm{0}) - \sum_{t=1}^T f_t^*(\bm{0})$; as shown before $h_{\delta}^*(\bm{0}) \leq D / \delta$ and $\forall \x \in \mathcal{K}, h_{\delta}^*(\x) \leq D / \delta$; according to Assumption \ref{assume:function}, for all $t \in [T], \forall \x \in \mathcal{K}, \| \nabla f_t(\x) \| \leq G$. Combining all these properties and choosing the same optimal value of the regularization parameter $\delta = 2 / G \sqrt{dT}$ concludes the stated bound \eqref{eq:newregretbnd}. The use of all the assumptions is identical to that of Theorem \ref{thm:eftpl} and detailed in Remark \ref{rem:assumptions}.
\end{proof}

All that remains to reach the conclusions by Theorem \ref{thm:main} is to use Lemma \ref{lemma:prelimregret} and handle the additional regret terms $\langle \tilde{\nabla}_t, \hat{\x}_t-\tilde{\x}_t \rangle$ for each $t \in [T]$. The following claims about smooth functions and empirical averages of random vectors are necessary for the latter part.

\begin{lemma}\label{lemma:smooth}
If $f : \mathcal{K} \to \R$ is a $\beta$-smooth function, then for any $x, y \in \mathcal{K}$
\begin{equation}
    \langle \nabla f(y), x - y \rangle \leq \langle \nabla f(x), x - y \rangle + \beta \| x-y \|^2
\end{equation}
\end{lemma}

\begin{lemma}\label{lemma:variance}
Let $Z_1, \dots, Z_m \sim \mathcal{Z}$ be i.i.d. samples of a bounded random vector $Z \in \R^d$, $\| Z \| \leq D$, with mean $\overline{Z} = \E[Z]$. Denote $\overline{Z}_m = \frac{1}{m} \sum_{j=1}^m Z_j$, then
\begin{equation}
    \E_{\mathcal{Z}} \left[ \| \overline{Z}_m - \overline{Z} \|^2 \right] \leq \frac{4 D^2}{m}
\end{equation}
\end{lemma}

\begin{proof}[Proof of Theorem \ref{thm:main}] First, note that according to Lemma \ref{lemma:variance}, the following bound $\E_{\xi_t}[ \| \hat{\x}_t-\tilde{\x}_t \|] \leq \sqrt{\E_{\xi_t}[ \| \hat{\x}_t-\tilde{\x}_t \|^2]} \leq \frac{2D}{\sqrt{m}}$ holds for all $t \in [T]$. In the case of general convex loss functions, use the Cauchy-Schwartz inequality along with the norm bound on the gradients and take expectation over the whole randomness in the algorithm $\xi_{1:T}$ in the reverse order $\xi_T, \dots, \xi_1$ to obtain for each $t \in [T]$
\begin{equation}\label{eq:cvxxtrabnd}
    \E_{\xi_{1:T}}[\langle \tilde{\nabla}_t, \hat{\x}_t-\tilde{\x}_t \rangle] \leq G \E_{\xi_{1:t}}[ \| \hat{\x}_t-\tilde{\x}_t \|] = G \E_{\xi_{1:t-1}}\left[ \E_{\xi_t}[ \| \hat{\x}_t-\tilde{\x}_t \| \, | \, \xi_{1:t-1}] \right] \leq \frac{2DG}{\sqrt{m}}
\end{equation}
Ordering the randomness of the iterations in reverse and taking the expectation conditional on $\xi_{1:t-1}$ is necessary in order to use Lemma \ref{lemma:variance} since $\hat{\x}_t$ is a deterministic quantity over $\xi_t$ only when conditioned on the previous randomness $\xi_{1:t-1}$. Finally, taking expectation over $\xi_{1:T}$ on the bound in \eqref{eq:newregretbnd} and using \eqref{eq:cvxxtrabnd} for all $t=1, \dots, T$ concludes the expected regret bound of $\E[\mathcal{R}_T(\mathcal{A}_1)] \leq \mathcal{R}_T(\mathcal{A}_2)+\frac{2DGT}{\sqrt{m}}$ given in detail in \eqref{eq:genregret} for general convex loss functions. It is worth to mention that this result did not require any assumptions on how the loss function $f_t(\cdot)$ at each iteration $t \in [T]$ is chosen by the adversary: in particular, the result holds for the strongest adaptive adversarial setting where the adversary can pick $f_t(\cdot)$ having knowledge of the previous actions by the player, i.e. the randomness $\xi_{1:t-1}$. This is true due to the fact that all the terms containing the function $f_t(\cdot)$ explicitly, e.g. $\tilde{\nabla}_t$, are separated and bound on their own.

The case of smooth convex loss functions requires a more nuanced approach in order to achieve an improvement on the general result. The key is to replace the gradient at the point $\tilde{\x}_t$ with a quantity that does not depend on $\xi_t$ and leverage the fact that $\tilde{\x}_t$ is an unbiased estimator of $\hat{\x}_t$. More formally, given $f_t(\cdot)$ is a $\beta$-smooth function use Lemma \ref{lemma:smooth} to get $\langle \tilde{\nabla}_t, \hat{\x}_t-\tilde{\x}_t \rangle \leq \langle \hat{\nabla}_t, \hat{\x}_t-\tilde{\x}_t \rangle + \beta \| \hat{\x}_t-\tilde{\x}_t \|^2$ where $\hat{\nabla}_t = \nabla f_t(\hat{\x}_t)$ is denoted accordingly. The quantities $f_t(\cdot)$ and $\hat{\x}_t$ are both (potentially) dependant on previous randomness $\xi_{1:t-1}$ but are deterministic with respect to $\xi_t$ when conditioned on $\xi_{1:t-1}$, hence so is $\hat{\nabla}_t$. Thus, it holds that $\E_{\xi_t}[\langle \hat{\nabla}_t, \hat{\x}_t-\tilde{\x}_t \rangle \, | \, \xi_{1:t-1}]] = 0$ for all $t \in [T]$. This fact results in the additional regret having a quadratic dependence on the estimation error instead of linear as before:
\begin{equation}\label{eq:smthxtrabnd}
    \E_{\xi_{1:T}}[\langle \tilde{\nabla}_t, \hat{\x}_t-\tilde{\x}_t \rangle] \leq \E_{\xi_{1:t-1}} \left[ \E_{\xi_{t}}[\langle \hat{\nabla}_t, \hat{\x}_t-\tilde{\x}_t \rangle + \beta \| \hat{\x}_t-\tilde{\x}_t \|^2 \, | \, \xi_{1:t-1}] \right] \leq \frac{4 \beta D^2}{m}
\end{equation}
Use the result obtained in \eqref{eq:smthxtrabnd} for all $t \in [T]$ in order to bound the additional regret term in \eqref{eq:newregretbnd} and conclude the expected regret bound of $\E[\mathcal{R}_T(\mathcal{A}_1)] \leq \mathcal{R}_T(\mathcal{A}_2) + \frac{4 \beta D^2 T}{m}$ given in detail in \eqref{eq:smoothregret} for smooth convex loss functions. Since the adversary is allowed to pick the loss function $f_t(\cdot)$ that depends on the previous randomness $\xi_{1:t-1}$, this regret bound again holds in the strongest adaptive adversarial setting.
\end{proof}

\subsection{Reduction to OSPF}\label{sec:reduction}
The results given in Theorem \ref{thm:main} indicate $O(\sqrt{T})$ optimal regret bounds for both convex and smooth convex loss functions when taking $m = O(T), O(\beta \sqrt{T})$ respectively, as suggested by Remark \ref{rem:regretbounds}. However, $m$ is not simply a parameter of the algorithm: it indicates the number of linear optimizations per iteration so in $T$ iterations the regret $O(\sqrt{T})$ is achieved with an overall linear optimization complexity of $m \cdot T$. To avoid such convoluted claims, we instead provide a reduction of Algorithm \ref{alg:main}, named OSPF in the smooth case, to the setting of one linear optimization per iteration that gives $O(T^{2/3})$ and $O(T^{3/4})$ expected regret for smooth and general convex losses, respectively.
\begin{proof}[Proof of Corollary \ref{cor:reduction}]
The reduction follows a simple blocking technique, i.e. grouping several rounds of the game into one as detailed in Algorithm \ref{alg:ospf}. Consider the online optimization setting with loss functions $f_1, \dots, f_T$ by the adversary after playing the actions $\x_1, \dots, \x_T$ using only one linear optimization per iteration. Let $T = nk$ where $n, k$ are assumed to be integers for simplicity and denote
\begin{equation}
f'_i = \sum_{t=(i-1)\cdot k+1}^{i\cdot k} f_t, \quad \forall i = 1, \dots, n    
\end{equation}
Since each $f'_i, i \in [n]$ contains $k$ losses from the original problem, then the player is allowed $k$ linear optimizations to handle a single loss $f'_i$. Hence, use Algorithm \ref{alg:main} for $T=n$ iterations with $m=k$ samples at each iteration to get actions $\x'_1, \dots, \x'_n$ and play $\x_t = \x'_i$ for all $(i-1) \cdot k + 1 \leq t \leq i \cdot k$ in the original setting -- call this algorithm $\mathcal{A}'_1$. The corresponding constants of the constructed game are $D' = D$ and $G' = G \cdot k$ since the constraint set $\mathcal{K}$ remains unchanged and a loss function constitutes $k$ original losses together. Thus, the expected regret bound of $\mathcal{A}'_1$ for general convex functions, according to Theorem \ref{thm:main}, is given by
\begin{equation}
    \E [\mathcal{R}_T(\mathcal{A}'_1)] \leq 2 D / \delta + \delta D (G \cdot k)^2 \cdot d n / 2 + \frac{2 (G \cdot k) D n}{\sqrt{k}} = 2DG\sqrt{d}\sqrt{n}k+2DGn\sqrt{k}
\end{equation}
with the parameter choice of $\delta = 2/G\sqrt{d}\sqrt{n}k$. Letting $n = k = T^{1/2}$ yields the expected regret bound for the algorithm $\mathcal{A}'_1$ that uses one linear optimization per iteration as $\E[\mathcal{R}_T(\mathcal{A}'_1)] = O(\sqrt{n}k+n\sqrt{k}) = O(T^{3/4})$ for general convex functions. The case of smooth convex functions is handled analogously. Note that the $\mathcal{A}'_1$ algorithm is equivalent to $\mathcal{A}_{\text{OSPF}}$ given in Algorithm \ref{alg:ospf}. The smoothness parameter of a sum of $k$ functions that are $\beta$-smooth equals $\beta' = \beta \cdot k$. Hence, the expected regret bound of $\mathcal{A}_{\text{OSPF}}$ for smooth convex functions is given by
\begin{equation}
    \E [\mathcal{R}_T(\mathcal{A}_{\text{OSPF}})] \leq 2 D / \delta + \delta D (G \cdot k)^2 \cdot d n / 2 + \frac{4 (\beta \cdot k) D^2 n}{k} = 2DG\sqrt{d}\sqrt{n}k + 4\beta D^2 n
\end{equation}
with the same choice of $\delta = 2/G\sqrt{d}\sqrt{n}k$. In this case let $n = T^{2/3}$ and $k = T^{1/3}$ to attain the expected regret bound  $O(T^{2/3})$ for $\mathcal{A}_{\text{OSPF}}$ with one linear optimization per iteration.
\end{proof}

\subsection{High Probability Bounds} \label{sec:highprob}
The theoretical guarantees for the main algorithm of this paper, Algorithm \ref{alg:main}, have all been in terms of expected regret as the performance metric. Even though expected regret is a widely accepted metric for online randomized algorithms, one might wonder whether the expectation bound holds only due to a balance of large and small chunks of regret or the given result actually holds most of the time. To answer this question, we provide bounds on $\mathcal{R}_T(\mathcal{A}_1)$ asymptotically equivalent (up to logarithmic factors) to the statements from Theorem \ref{thm:main} that hold with high probability over the randomness in $\mathcal{A}_1$: these results also transfer analogously to the reduction from section \ref{sec:reduction}. In particular, the following theorem shows that Algorithm \ref{alg:main} obtains regret of $\mathcal{R}_T(\mathcal{A}_1) = O(\sqrt{T})+\tilde{O}(T/\sqrt{m})$ for general convex loss functions and $\mathcal{R}_T(\mathcal{A}_1) = O(\sqrt{T})+\tilde{O}(\beta T/m)$ for smooth convex loss functions both holding with high probability.
\begin{theorem}\label{thm:highprob}
Given that the Assumptions \ref{assume:set} and \ref{assume:function} hold, the regret of Algorithm \ref{alg:main} for general convex loss functions is w.p. $1-\sigma$ for any $\sigma>0$ bounded by
\begin{equation}\label{eq:highprobconvex}
    \sum_{t=1}^T f_t(\tilde{\x}_t) \leq \min_{\x \in \mathcal{K}} \{ \sum_{t=1}^T f_t(\x) \} + \mathcal{R}_T(\mathcal{A}_3) + \frac{2 G D T}{\sqrt{m}} \cdot \sqrt{\log 2T / \sigma}
\end{equation}
If the convex loss functions are also $\beta$-smooth, then it is w.p. $1-\sigma$ for any $\sigma > 0$ bounded by
\begin{equation}\label{eq:highprobsmooth}
    \sum_{t=1}^T f_t(\tilde{\x}_t) \leq \min_{\x \in \mathcal{K}} \{ \sum_{t=1}^T f_t(\x) \} + \mathcal{R}_T(\mathcal{A}_3) + 2 G D \sqrt{2 T \log 4 / \sigma} + \frac{8 \beta D^2 T}{m} \cdot \log 4T / \sigma
\end{equation}
\end{theorem}

\section{Discussion}
We have presented an efficient projection-free method for online convex optimization with smooth functions that makes only a single linear optimization computation per iteration and achieves regret $T^{2/3}$, improving upon the previous bound of $T^{3/4}$.

Certain algorithms in the literature make more than one linear optimization computation per iteration. 
To make  the comparison to other methods more precise, we need a more refined computational metric. Define the following complexity metric for an online projection-free algorithm: let $\A$ be an online optimization algorithm, and define $T_\eps(\A)$ to be the overall number of gradient oracle calls as well as linear optimization calls made until the  average regret becomes at most $\eps$. 

In these terms, we have shown an algorithm with $T_\eps = O(\frac{d}{\eps^3})$ for smooth functions, as compared to $O(\frac{1}{\eps^4})$ which is the previous best.

It thus remains open to obtain a $\eps^{-3}$-complexity algorithm for general convex sets that does not depend on the dimension, or show that this is impossible. It is also unknown at this time if these dependencies on $\eps$, in both the smooth and non-smooth cases, are tight.

\bibliographystyle{plainnat}
\bibliography{ref}

\begin{thebibliography}{24}
\providecommand{\natexlab}[1]{#1}
\providecommand{\url}[1]{\texttt{#1}}
\expandafter\ifx\csname urlstyle\endcsname\relax
  \providecommand{\doi}[1]{doi: #1}\else
  \providecommand{\doi}{doi: \begingroup \urlstyle{rm}\Url}\fi

\bibitem[Allen-Zhu et~al.(2017)Allen-Zhu, Hazan, Hu, and Li]{allen2017linear}
Zeyuan Allen-Zhu, Elad Hazan, Wei Hu, and Yuanzhi Li.
\newblock Linear convergence of a frank-wolfe type algorithm over trace-norm
  balls.
\newblock In \emph{Advances in Neural Information Processing Systems}, pages
  6191--6200, 2017.

\bibitem[Argyriou et~al.(2014)Argyriou, Signoretto, and
  Suykens]{argyriou2014hybrid}
Andreas Argyriou, Marco Signoretto, and Johan Suykens.
\newblock Hybrid conditional gradient - smoothing algorithms with applications
  to sparse and low rank regularization, 2014.

\bibitem[Chen et~al.(2018)Chen, Harshaw, Hassani, and
  Karbasi]{chen18csubmodular}
Lin Chen, Christopher Harshaw, Hamed Hassani, and Amin Karbasi.
\newblock Projection-free online optimization with stochastic gradient: From
  convexity to submodularity.
\newblock In Jennifer Dy and Andreas Krause, editors, \emph{Proceedings of the
  35th International Conference on Machine Learning}, volume~80 of
  \emph{Proceedings of Machine Learning Research}, pages 814--823,
  Stockholmsmässan, Stockholm Sweden, 10--15 Jul 2018. PMLR.
\newblock URL \url{http://proceedings.mlr.press/v80/chen18c.html}.

\bibitem[Frank and Wolfe(1956)]{FrankWo56}
Marguerite Frank and Philip Wolfe.
\newblock An algorithm for quadratic programming.
\newblock \emph{Naval research logistics quarterly}, 3\penalty0 (1-2):\penalty0
  95--110, 1956.

\bibitem[Garber(2016)]{garber2016faster}
Dan Garber.
\newblock Faster projection-free convex optimization over the spectrahedron.
\newblock In \emph{Advances in Neural Information Processing Systems}, pages
  874--882, 2016.

\bibitem[Garber and Hazan(2013)]{GarberHa13}
Dan Garber and Elad Hazan.
\newblock A linearly convergent conditional gradient algorithm with
  applications to online and stochastic optimization.
\newblock \emph{arXiv preprint arXiv:1301.4666}, 2013.

\bibitem[Hassani et~al.(2019)Hassani, Karbasi, Mokhtari, and
  Shen]{hassani2019stochastic}
Hamed Hassani, Amin Karbasi, Aryan Mokhtari, and Zebang Shen.
\newblock Stochastic conditional gradient++, 2019.

\bibitem[Hazan(2008)]{hazan2008sparse}
Elad Hazan.
\newblock Sparse approximate solutions to semidefinite programs.
\newblock In \emph{Latin American symposium on theoretical informatics}, pages
  306--316. Springer, 2008.

\bibitem[Hazan and Kale(2012)]{HazanKa12}
Elad Hazan and Satyen Kale.
\newblock Projection-free online learning.
\newblock In \emph{Proceedings of the 29th International Conference on Machine
  Learning}, 2012.

\bibitem[Hazan and Luo(2016)]{hazan2016variance}
Elad Hazan and Haipeng Luo.
\newblock Variance-reduced and projection-free stochastic optimization.
\newblock In \emph{International Conference on Machine Learning}, pages
  1263--1271, 2016.

\bibitem[Hazan et~al.(2016)]{hazan2016introduction}
Elad Hazan et~al.
\newblock Introduction to online convex optimization.
\newblock \emph{Foundations and Trends{\textregistered} in Optimization},
  2\penalty0 (3-4):\penalty0 157--325, 2016.

\bibitem[Jaggi(2013)]{Jaggi13}
Martin Jaggi.
\newblock Revisiting frank-wolfe: Projection-free sparse convex optimization.
\newblock In \emph{Proceedings of the 30th International Conference on Machine
  Learning}, pages 427--435, 2013.

\bibitem[Kalai and Vempala(2005)]{kalai2005efficient}
Adam Kalai and Santosh Vempala.
\newblock Efficient algorithms for online decision problems.
\newblock \emph{Journal of Computer and System Sciences}, 71\penalty0
  (3):\penalty0 291--307, 2005.

\bibitem[Lan(2013)]{lan2013complexity}
Guanghui Lan.
\newblock The complexity of large-scale convex programming under a linear
  optimization oracle, 2013.

\bibitem[Lan and Zhou(2016)]{LanZh14}
Guanghui. Lan and Yi. Zhou.
\newblock Conditional gradient sliding for convex optimization.
\newblock \emph{SIAM Journal on Optimization}, 26\penalty0 (2):\penalty0
  1379--1409, 2016.
\newblock \doi{10.1137/140992382}.
\newblock URL \url{https://doi.org/10.1137/140992382}.

\bibitem[Levy and Krause(2019)]{pmlr-v89-levy19a}
Kfir Levy and Andreas Krause.
\newblock Projection free online learning over smooth sets.
\newblock In Kamalika Chaudhuri and Masashi Sugiyama, editors,
  \emph{Proceedings of Machine Learning Research}, volume~89 of
  \emph{Proceedings of Machine Learning Research}, pages 1458--1466. PMLR,
  16--18 Apr 2019.
\newblock URL \url{http://proceedings.mlr.press/v89/levy19a.html}.

\bibitem[Merhav et~al.(2002)Merhav, Ordentlich, Seroussi, and
  Weinberger]{merhav2002sequential}
Neri Merhav, Erik Ordentlich, Gadiel Seroussi, and Marcelo~J Weinberger.
\newblock On sequential strategies for loss functions with memory.
\newblock \emph{IEEE Transactions on Information Theory}, 48\penalty0
  (7):\penalty0 1947--1958, 2002.

\bibitem[Pierucci et~al.(2014)Pierucci, Harchaoui, and
  Malick]{pierucci2014smoothing}
Federico Pierucci, Zaid Harchaoui, and J{\'e}r{\^o}me Malick.
\newblock {A smoothing approach for composite conditional gradient with
  nonsmooth loss}.
\newblock Research Report RR-8662, {INRIA Grenoble}, July 2014.
\newblock URL \url{https://hal.inria.fr/hal-01096630}.

\bibitem[Pinelis(1994)]{pinelis1994martingale}
Iosif Pinelis.
\newblock Optimum bounds for the distributions of martingales in banach spaces.
\newblock \emph{The Annals of Probability}, 22\penalty0 (4):\penalty0
  1679--1706, 1994.
\newblock ISSN 00911798.

\bibitem[Shalev-Shwartz and Singer(2007)]{shalev2007primal}
Shai Shalev-Shwartz and Yoram Singer.
\newblock A primal-dual perspective of online learning algorithms.
\newblock \emph{Machine Learning}, 69\penalty0 (2-3):\penalty0 115--142, 2007.

\bibitem[Xie et~al.(2019)Xie, Shen, Zhang, Wang, and Qian]{xie2019efficient}
Jiahao Xie, Zebang Shen, Chao Zhang, Boyu Wang, and Hui Qian.
\newblock Efficient projection-free online methods with stochastic recursive
  gradient, 2019.

\bibitem[Yurtsever et~al.(2019)Yurtsever, Sra, and
  Cevher]{yurtsever19conditional}
Alp Yurtsever, Suvrit Sra, and Volkan Cevher.
\newblock Conditional gradient methods via stochastic path-integrated
  differential estimator.
\newblock In Kamalika Chaudhuri and Ruslan Salakhutdinov, editors,
  \emph{Proceedings of the 36th International Conference on Machine Learning},
  volume~97 of \emph{Proceedings of Machine Learning Research}, pages
  7282--7291, Long Beach, California, USA, 09--15 Jun 2019. PMLR.
\newblock URL \url{http://proceedings.mlr.press/v97/yurtsever19b.html}.

\bibitem[Zhang et~al.(2019)Zhang, Shen, Mokhtari, Hassani, and
  Karbasi]{zhang2019sample}
Mingrui Zhang, Zebang Shen, Aryan Mokhtari, Hamed Hassani, and Amin Karbasi.
\newblock One sample stochastic frank-wolfe, 2019.

\bibitem[Zinkevich(2003)]{zinkevich2003online}
Martin Zinkevich.
\newblock Online convex programming and generalized infinitesimal gradient
  ascent.
\newblock In \emph{Proceedings of the 20th international conference on machine
  learning (icml-03)}, pages 928--936, 2003.

\end{thebibliography}

\appendix

\section{Proof of the high probability bounds}

This section focuses on regret bound results for Algorithm \ref{alg:main} that hold with high probability. We use the following Azuma-type concentration inequality for vector-valued martingales derived as an application of the work by \cite{pinelis1994martingale} to the Euclidean space $\R^d$ .

\begin{proposition}[Theorem 3.5 in \cite{pinelis1994martingale}]\label{app:prop}
Let $\bm{\nu}_1, \dots, \bm{\nu}_K \in \R^d$ be a vector-valued martingale difference with respect to $\{\mathcal{F}_k\}_{k=1}^K$ such that for all $k \in [K]$ it holds that $\E[\bm{\nu}_k | \mathcal{F}_{k-1}] = \bm{0}$ and $\|\bm{\nu}_k\| \leq c_k$ for some $c_k > 0$. Then for any $\lambda > 0$
\begin{equation}
    \Pr \left[ \left \| \sum_{k=1}^K \bm{\nu}_k \right \| > \lambda \right] \leq 2 \exp \left( - \frac{\lambda^2}{2 \sum_{k=1}^K c_k^2} \right)
\end{equation}
\end{proposition}

\begin{proof}[Proof of Theorem \ref{thm:highprob}]
We first obtain proximity of the estimates $\tilde{\x}_t$ to their mean $\hat{\x}_t$ for all $t \in [T]$ that hold with high probability using Proposition \ref{app:prop}. Fix an arbitrary $t \in [T]$ and denote $\bm{\nu}_j = \frac{1}{m} (\hat{\x}_t - \x_t^j)$ for $j = 1, \dots, m$. Note that $\sum_{j=1}^m \bm{\nu}_j = \hat{\x}_t - \tilde{\x}_t$ and for each $j \in [m]$ we have $\E_{\v_t^j}[\bm{\nu}_j \, | \, \v_{t}^{1:j-1}] = \bm{0}$ given the definition of $\hat{\x}_t$ and i.i.d. uniform samples $\v_t^j \sim \mathbb{B}$, $j \in [m]$. Furthermore, $\| \bm{\nu}_j \| \leq 2 D / m$ using triangle inequality since $\x_t^j, \hat{\x}_t \in \mathcal{K}$ given convexity of the constraint set $\mathcal{K}$. Fix any $\sigma > 0$ and let $\lambda = \frac{2D}{\sqrt{m}} \cdot \sqrt{2 \log 2 T / \sigma}$, then by Proposition \ref{app:prop}
\begin{equation}\label{app:eq:estimate}
    \Pr_{\xi_t} \left[ \| \hat{\x}_t - \tilde{\x}_t \| \geq \lambda \right] \leq \frac{\sigma}{T} \implies \Pr_{\xi_{1:T}} \left[ \forall t \in [T], \, \| \hat{\x}_t - \tilde{\x}_t \| \geq \lambda \right] \leq \sigma
\end{equation}
where the implication stems from union bound. To conclude the regret bound for general convex functions with high probability, use Lemma \ref{lemma:prelimregret} and apply $\langle \tilde{\nabla}_t, \hat{\x}_t - \tilde{\x}_t \rangle \leq G \| \hat{\x}_t - \tilde{\x}_t \|$ to get that
\begin{equation}
    \mathcal{R}_T(\mathcal{A}_1) \leq \mathcal{R}_T(\mathcal{A}_3) + \lambda G T = \mathcal{R}_T(\mathcal{A}_3) + \frac{2 G D T}{\sqrt{m}} \cdot \sqrt{2 \log 2T / \sigma}
\end{equation}
holds with probability at least $1-\sigma$. In the smooth convex case, proceed analogously and apply the inequality $\langle \tilde{\nabla}_t, \hat{\x}_t - \tilde{\x}_t \rangle \leq \langle \hat{\nabla}_t, \hat{\x}_t - \tilde{\x}_t \rangle + \beta \| \hat{\x}_t - \tilde{\x}_t \|^2$ given by Lemma \ref{lemma:smooth}. Fix any $\sigma > 0$ and denote $\zeta_t = \langle \hat{\nabla}_t, \hat{\x}_t - \tilde{\x}_t \rangle$ for all $t \in [T]$. Notice that $\{ \zeta_t \}_{t=1}^T$ is a martingale difference with respect to $\xi_{1:T}$. Indeed, $\E_{\xi_t}[\tilde{\x}_t \, | \, \xi_{1:t-1}] = \hat{\x}_t$ and the quantities $f_t(\cdot), \hat{\x}_t$, and hence $\hat{\nabla}_t$, are deterministic given $\xi_{1:t-1}$ which means that $\E_{\xi_t}[\zeta_t \, | \, \xi_{1:t-1}] = 0$. Moreover, they are bounded $|\zeta_t| \leq \| \hat{\nabla}_t \| \cdot \| \hat{\x}_t - \tilde{\x}_t \| \leq 2 G D = c_t$ using Cauchy-Schwartz inequality, triangle inequality and convexity of $\mathcal{K}$. Letting $\gamma = 2 G D \sqrt{2 T \log 4 / \sigma}$ Azuma's inequality yields
\begin{equation}\label{app:eq:martingale}
    \Pr_{\xi_{1:T}} \left[ \left \lvert \sum_{t=1}^T \zeta_t \right \rvert \geq \gamma \right] \leq 2 \exp \left( -\frac{\gamma^2}{2 \sum_{t=1}^T c_t^2} \right) = \sigma / 2
\end{equation}
Combine \eqref{app:eq:martingale} and the already obtained \eqref{app:eq:estimate} with $\sigma' = \sigma / 2$, and corresponding $\lambda'$, to conclude the regret bound for smooth convex functions using union bound and the Lemma \ref{lemma:prelimregret} to obtain that
\begin{equation}\label{app:eq:smoothhp}
    \mathcal{R}_T(\mathcal{A}_1) \leq \mathcal{R}_T(\mathcal{A}_3) + \gamma + \beta (\lambda')^2 T = \mathcal{R}_T(\mathcal{A}_3) + 2 G D \sqrt{2 T \log 4/ \sigma} + \frac{8 \beta D^2 T}{m} \cdot \log 4 T / \sigma
\end{equation}
holds with probability at least $1-\sigma$. This finishes the proof of Theorem \ref{thm:highprob}. The bound in \eqref{app:eq:smoothhp} implies, following the same logic as in section \ref{sec:reduction}, that the regret bound $\mathcal{R}_T(\mathcal{A}_{\text{OSPF}}) = \tilde{O}(T^{2/3} \log 1/ \sigma)$ holds with high probability $1-\sigma$.
\end{proof}

\section{Miscellaneous proofs}
\begin{proof}[Proof of Lemma \ref{lemma:lipschitz}]
To show convexity, consider arbitrary $\y_1, \y_2 \in \R^d$ and $\lambda > 0$, denote $\y_{12} = \lambda \y_1 + (1-\lambda) \y_2$. Then
\begin{equation*}
    \mathcal{M}_{\mathcal{K}}(\y_{12}) = \langle \y_{12}, \x^* \rangle = \lambda \langle \y_1, \x^* \rangle + (1-\lambda) \langle \y_2, \x^* \rangle \leq \lambda \mathcal{M}_{\mathcal{K}}(\y_1) + (1-\lambda) \mathcal{M}_{\mathcal{K}}(\y_2)
\end{equation*}
Next, fix arbitrary $\y_1, \y_2 \in \R^d$ and suppose w.l.o.g. that $\mathcal{M}_{\mathcal{K}}(\y_1) \geq \mathcal{M}_{\mathcal{K}}(\y_2)$. Then
\begin{align*}
    \mathcal{M}_{\mathcal{K}}(\y_1) - \mathcal{M}_{\mathcal{K}}(\y_2) &= \langle \y_1, \mathcal{O}_{\mathcal{K}}(\y_1) \rangle - \langle \y_2, \mathcal{O}_{\mathcal{K}}(\y_2) \rangle \\
    &\leq \langle \y_1, \mathcal{O}_{\mathcal{K}}(\y_1) \rangle - \langle \y_2, \mathcal{O}_{\mathcal{K}}(\y_1) \rangle \\
    &= \langle \mathcal{O}_{\mathcal{K}}(\y_1), \y_1 - \y_2 \rangle \leq \| \mathcal{O}_{\mathcal{K}}(\y_1)\| \| \y_1 - \y_2 \| \leq D \| \y_1 - \y_2 \|
\end{align*}
where the first inequality follows from the definition of $\mathcal{O}_{\mathcal{K}}(\y_2)$ while the rest is achieved using the Cauchy-Schwarz inequality and the norm bound of the constraint set $\mathcal{K}$.
\end{proof}

\begin{proof}[Proof of Lemma \ref{lemma:randsmoothing}]
According to Stokes' theorem, the gradient of the smoothed function $\hat{g}(\cdot)$ can be written as
\begin{equation*}
    \nabla \hat{g}(\y) = \delta d \E_{\v \sim \mathbb{S}} \left[ g(\y + \frac{1}{\delta} \v) \v \right]
\end{equation*}
where $\mathbb{S} = \{\v \in \R^d, \| \v \| = 1 \}$ denotes the unit sphere, the boundary of $\mathbb{B}$. Then for arbitrary $\y_1, \y_2 \in \R^d$
\begin{equation*}
    \| \nabla \hat{g}(\y_1) - \nabla \hat{g}(\y_2) \| = \delta d \left \| \E_{\v \sim \mathbb{S}} \left[ g(\y_1 + \frac{1}{\delta} \v) \v - g(\y_2 + \frac{1}{\delta} \v) \v \right] \right \| \leq \delta d L \| \y_1 - \y_2 \|
\end{equation*}
using linearity of expectation, Jensen's inequality and the Lipschitz property of $g(\cdot)$. It follows that $\hat{g}(\cdot)$ is a $\delta d L$-smooth function.
\end{proof}

\begin{proof}[Proof of Lemma \ref{lemma:smooth}]
The function $f : \mathcal{K} \to \R$ being $\beta$-smooth is equivalent to its gradient being $\beta$-Lipschitz, hence
\begin{equation*}
    \langle \nabla f(y) - \nabla f(x), x - y \rangle \leq \| \nabla f(y) - \nabla f(x) \| \cdot \| x - y \| \leq \beta \| x - y \|^2
\end{equation*}
The desired inequality follows from the result above.
\end{proof}

\begin{proof}[Proof of Lemma \ref{lemma:variance}]
Given that $\| Z \| \leq D$ and $\overline{Z} = \E[Z]$ we have that $\| \overline{Z} \| \leq D$ and $\| Z - \overline{Z} \| \leq \| Z \| + \| \overline{Z} \| \leq 2D$. Hence, by linearity of expectation and linearity of variance for independent random variables we obtain
\begin{align*}
    \E_{\mathcal{Z}}[\| \overline{Z}_m - \overline{Z} \|^2] &= \E_{\mathcal{Z}} \left[ \sum_{i=1}^d (\overline{Z}_m(i) - \overline{Z}(i))^2 \right] = \sum_{i=1}^d \E_{\mathcal{Z}_i}[(\overline{Z}_m(i)-\overline{Z}(i))^2] = \\
    &= \sum_{i=1}^d \var_{\mathcal{Z}_i}(\overline{Z}_m(i)) = \sum_{i=1}^d \frac{1}{m^2} \sum_{j=1}^m \var_{\mathcal{Z}_i}(Z_j(i)) = \frac{1}{m} \sum_{i=1}^d \var_{\mathcal{Z}_i}(Z(i)) = \\
    &= \frac{1}{m} \sum_{i=1}^d \E_{\mathcal{Z}_i} [(Z(i) - \overline{Z}(i))^2] = \frac{1}{m} \E_{\mathcal{Z}}[\| Z-\overline{Z} \|^2] \leq \frac{4D^2}{m}
\end{align*}
\end{proof}

\end{document}